\documentclass[draftclsnofoot,peerreview,onecolumn]{IEEEtran}

\usepackage{cite}
\usepackage{url}

\usepackage{amssymb}

\usepackage{amsmath}
\usepackage{amsthm}

\usepackage[noend]{algorithmic}

\usepackage{graphicx}

\ifCLASSOPTIONcompsoc
\usepackage[caption=false,font=normalsize,labelfon
t=sf,textfont=sf]{subfig}
\else
\usepackage[caption=false,font=footnotesize]{subfig}
\fi

\usepackage{array}

\usepackage{bbm}

\usepackage{balance}

\newtheorem{theorem}{Theorem}
\newtheorem{lemma}{Lemma}
\newtheorem{corollary}{Corollary}

\newtheorem{remark}{Remark}

\DeclareMathAlphabet{\mathcurve}{OMS}{cmsy}{m}{n}

\newcommand{\define}{\overset{\triangle}{=}}
\newcommand{\argmin}[1]{\operatorname{arg}\min_{#1}}

\newcommand{\E}[2]{\mathbb{E}_{#1}\left[#2\right]}

\DeclareMathOperator{\project}{Proj}
\newcommand{\proj}[2]{\project_#1\left(#2\right)}

\newcommand{\reals}{\mathbb{R}}
\newcommand{\dimcnt}{N}
\newcommand{\set}[1]{\mathcurve{#1}}
\newcommand{\Set}{\set{K}}
\newcommand{\subgrads}[2]{\partial #1(#2)}

\newcommand{\loss}{g}

\newcommand{\decision}{w}
\newcommand{\altdec}{v}

\newcommand{\observe}{\widetilde{\loss}}
\newcommand{\prob}{p}

\newcommand{\noise}{\gamma}

\newcommand{\pdf}{\mathbb{P}}

\newcommand{\deri}{\mathrm{d}}

\newcommand{\regret}{R_T}
\newcommand{\regretK}{R_K}

\begin{document}
\title{
	Minimax Optimal Online Stochastic Learning for Sequences of Convex Functions under Sub-Gradient Observation Failures}
\author{
	Hakan~Gokcesu, 
	and~Suleyman~S.~Kozat,~\IEEEmembership{Senior~Member,~IEEE}
	\thanks{
		H. Gokcesu and Suleyman S. Kozat are with the Department of Electrical and Electronics Engineering, Bilkent University, Ankara, Turkey; e-mail: \{hgokcesu, kozat\}@ee.bilkent.edu.tr, tel: +90 (312) 290-2336.
	}
	\footnote{
		\copyright˜2019 	IEEE.  Personal use of this material is permitted.  Permission from IEEE must be obtained for all other uses, in any current or future media, including reprinting/republishing this material for advertising or promotional purposes, creating new collective works, for resale or redistribution to servers or lists, or reuse of any copyrighted component of this work in other works.
	}
}
\maketitle

\begin{abstract}
	We study online convex optimization under stochastic sub-gradient observation faults, where we introduce adaptive algorithms with minimax optimal regret guarantees.
	We specifically study scenarios where our sub-gradient observations can be noisy or even completely missing in a stochastic manner.
	To this end, we propose algorithms based on
	sub-gradient descent method, which achieve tight minimax optimal regret bounds. When necessary, these algorithms utilize properties of the underlying stochastic settings to optimize their learning rates (step sizes). These optimizations are the main factor in providing the minimax optimal performance guarantees, especially when observations are stochastically missing. However, in real world scenarios,
	these properties of the underlying stochastic settings may not be revealed to the optimizer. 
	For such a scenario, we propose a blind algorithm 
	that estimates these properties empirically in a generally applicable manner. Through extensive experiments, we show that this empirical approach is a natural combination of regular stochastic gradient descent and the minimax optimal algorithms (which work best for randomized and adversarial function sequences, respectively).
\end{abstract}


\newpage\section{Introduction}
In online learning, a parameter vector of interest is optimized sequentially based on the feedback coming from the environment. 
A widely used approach for this purpose
is to employ iterative decision update methods such as online gradient descent. These approaches are shown to have worst-case optimal performance under convex target functions when the sub-gradients are fully observed and utilized 
\cite{Shalev,Hazan}. 
However, a substantial number of real-life optimization problems 
have limitations in their learning procedures regarding their observations of the environment vectors, which can be erroneous or noisy -especially in 
applications concerning compressed or privatized data- such as matrix recovery \cite{Gao}, signal reconstruction \cite{Lu}, parameter estimation \cite{Pakrooh}, time-frequency signature extraction \cite{Jokanovic}. 
Furthermore, some environmental data may remain unobserved or missing. Such can be the case in $H_{\infty}$ control/filtering \cite{Wang,Shen_H} and ARMA modeling \cite{Giannakis} with missing measurements, or
big data applications of low-dimensional sketching \cite{Shen_big}, data completion \cite{Mardani} and anomaly detection \cite{gokcesuAnomaly}.

To this end, we investigate the problem of learning under faulty feedback scenarios where environment vectors stochastically fail to be observed or the observations incorporate arbitrary zero-mean noise, for which we provide algorithms with worst-case optimal guarantees. In our solutions, we only assume the target functions to be convex, which does not necessitate additional properties such as smoothness or strong convexity, thus making them widely applicable. 
Our learning setting can be aptly called online stochastic convex optimization as the limitations in feedback brings stochasticity to our decision updates. 
In this work, we aim to construct a general convex optimization framework, which permits stochastic observation errors. This construction is applicable to a myriad of scenarios after the target functions (losses, errors, costs etc.) are modeled to be convex. Moreover, we demonstrate how our approaches provide tight optimality guarantees, making them 
robust to the observations (incoming data stream). 

\subsection{Prior Art and Comparisons}
In learning over time-series, 
a case of limited feedback scenario is called prediction under missing data, where the samples are placed apart with non-uniform intervals (e.g., aperiodic sampling). For this purpose, a specific linear filtering model of auto-regression is utilized in \cite{Anava}, which imputes its estimates to substitute for the missing data. They obtain performance guarantees 
with respect to the times where data is available. As opposed to this, we propose guarantees for the whole time horizon 
in an expected sense. 
Furthermore, another case of limited feedback setting is to learn under missing features as in \cite{Afshin}. The authors point out a common real world example of data corruption caused by sensor failure and consider feature grouping (like an imputation of averages) to compensate for information deficiency. However, unlike our method, this approach has no 
worst-case guarantees. 

Other topics of interest in the limited feedback setting are the active learning framework where the user has control over the observability of data \cite{Hanneke,gokcesuBandit}, the deterministic limited information setting \cite{Alon}, partial monitoring \cite{Cesa-Bianchi}, limited-attribute observation model \cite{Bullins} and bandit approximation for gradients \cite{Flaxman}. All such problems work under the limitation of learning with faulty or partial feedback, and our work complements these 
with its setting of stochastic observation faults. 
Prior works in this direction include 
the probabilistic optimization in \cite{Hennig}, which imposes the assumption of independent zero-mean Gaussian noises on each data coordinate. This assumption is much stronger than our zero-mean noise model. 
Another work is optimization with biased noisy gradient oracles in \cite{Hu}. It requires both smooth target functions 
and their repeated querying, which is possibly infeasible when each evaluation is costly, as is the case in most real life online learning applications. In contrast, our approach only needs the functions to be convex and evaluates each of them once.

We study the problem of limited information feedback with unbiased stochastic sub-gradient observation faults. Unlike previous approaches, the target convex functions need not display additional properties such as strong-convexity or Lipschitz-smoothness. Moreover, in the worst-case, 
we would not be allowed repeated evaluation of a noisy sub-gradient from a fixed function under this possibly chaotic or even adversarial setting. Our approach of optimization is efficient with constant memory and per-round time complexities and need not assume any nature on the observation noise. In this scenario, we may observe noisy environmental vectors (i.e. corrupted sub-gradients) corresponding to our decisions at each optimization round, or may even probabilistically fail to observe at all (i.e. missing sub-gradients). Our results are minimax optimal such that the lower and upper bounds of the difference between cumulative losses of our decisions and the best decision in hindsight is at most a constant factor in the worst-case scenario, regardless of possibly arbitrarily large decision set or sub-gradient norms.

\subsection{Contributions}
Our contributions are listed as follows: 
\begin{itemize}
	\item We introduce an adaptive online learning algorithm for convex functions, which is able to utilize sub-gradient feedback with arbitrary zero-mean noises and achieve minimax optimal regret guarantees.
	\item 
	Our algorithm works in general stochastic observability schemes with arbitrary conditional observation probabilities, which may also be randomly generated by arbitrary priors. We also prove the minimax optimality of the regret achieved by our design under this general stochastically missing sub-gradient feedback structure.
	\item The minimax optimality guarantees are shown for both first and second moments of the regret. These guarantees 
	optimally bound the mean and variance of our redundant loss against the best decision in hindsight, hence, introducing further stability to our performance.
	\item We improve upon \cite{Flaxman} and \cite{He} by introducing minimax optimal sub-gradient norm adaptivity for unbiased noisy gradient descent and eliminating the requirement of a constant probability for stochastic observation failures.
	\item 
	We further introduce an empirical method when probabilities to observe, or their priors, are inaccessible. This method naturally combines the standard stochastic gradient descent and our minimax approaches such that its applicability in the absence of knowledge about the stochastic setting is validated with detailed experiments on both adversarial and randomized function sequences.
\end{itemize} 

\subsection{Organization}
The rest of this paper is constructed as follows. 
In Section \ref{sec:oll_problem} and \ref{section:adagrad},
we formally define the problem and show 
an efficient minimax optimal adaptive projected sub-gradient descent method. 
In Section \ref{section:unbiasedgrad}, we show that the regret of this algorithm has minimax optimal 
first and second moments 
when the sub-gradients are observed with an additive zero-mean noise. 
In Section \ref{section:missing_gradient}, we show how to utilize the same algorithm when we may fail to make sub-gradient observations with a known probability of observation. 
In Section \ref{section:prior}, we consider the case when we only know a prior distribution, which randomly generates a probability to observe after each observation. In Section \ref{section:empirical}, we introduce a completely online 
and 
adaptive empirical method 
for when 
both conditional probabilities and even corresponding priors are unknown. Finally, in Section \ref{section:experiments} and \ref{sec:conclusion}, we present the experimental results 
and finish with concluding remarks.

\section{Problem Formulation and Notation} \label{sec:oll_problem}
We have a sequence of convex functions $f_t: \Set \rightarrow \reals$ for integers times $t \geq 1$, where $\Set$ is a convex subspace of $\reals^{\dimcnt}$, i.e. for all $\decision,\altdec \in \Set$, $(\lambda \decision + (1-\lambda)\altdec) \in \Set$ for any $0\leq \lambda\leq 1$. 
Members of the set $\Set$ are column vectors. For $\decision,\altdec \in \Set$, their inner product is $\decision^T \altdec$ where $\decision^T$ is the transpose of $\decision$. The norm (euclidean) of $\decision\in\Set$ is $\|\decision\| = \sqrt{\decision^T \decision} \geq 0$. We also define the projection operation $\proj{\Set}{\decision}$, which solves $\argmin{\altdec \in \Set} \|\decision-\altdec\|^2$, a relatively simple computation when $\Set$ is a hyper-ellipsoid or hyper-rectangle.  

Convexity of each $f_t(\cdot)$ implies the first-order inequality
\begin{equation} \label{eq:convexity}
f_t(\decision) - f_t(\altdec) \leq \loss_t^T (\decision - \altdec)
\end{equation}
for every pair $\decision,\altdec \in \Set$ and every sub-gradient 
$\loss_t \in \subgrads{f_t}{\decision}$.

The \emph{regret}, denoted as $\regret$, is defined as
\begin{equation} \label{eq:linearized_regret}
\regret
\define \sum_{t=1}^T f_t(\decision_t) - f_t(\decision^*) \leq \sum_{t=1}^T \loss_t^T (\decision_t - \decision^*),
\end{equation}
where $\{\decision_t\}_{t=1}^T$ is the decision sequence produced by the optimizer and $\decision^*$ is the best fixed decision in hindsight. The inequality comes from \eqref{eq:convexity} for any $\loss_t \in \subgrads{f_t}{\decision_t}$. This is a tight bound when no further assumptions are made about $f_t(\cdot)$ 
and holds with equality for linear functions $f_t(\decision) = \langle \loss_t, \decision \rangle$. 

In stochastic feedback setting, the decision sequence $\{\decision_t\}_{t=1}^T$ is determined with information exceedingly limited compared to 
observing $\loss_t \in \subgrads{f_t}{\decision_t}$. We may observe a noisy version $\observe_t$ with only its conditional expectation -given all past functions, decisions, feedbacks, underlying stochastic properties etc.- equaling $\loss_t$, i.e. $\E{t}{\observe_t} = \loss_t$. Moreover, it is also not guaranteed that we make an observation at every $t$. Consequently, also considering the worst-case scenario by an 
arbitrary function sequence, 
we want to optimally update our decisions after each observation, possibly corrupted by noise.

\renewcommand{\figurename}{Alg.}
\begin{figure}[!t]
	\centering
	\caption{Online Sub-Gradient Descent with Projection}\label{algorithm:gradient_descent}
	\begin{algorithmic}[1]
		\REQUIRE $\decision_1 \in \Set$, $\loss_t \in \subgrads{f_t}{\decision_t}$ for $t\geq1$.
		\ENSURE $\decision_t \in \Set$ for $t \geq 2$.
		\STATE Initialize $t=1$. 
		\WHILE{not terminated}
		\STATE Observe $\loss_t \in \subgrads{f_t}{\decision_t}$. 
		\STATE Decide $\eta_t$.
		\STATE Set $\decision_{t+1} = \proj{\Set}{\decision_t - \eta_t \loss_t}$. 
		\STATE $t \leftarrow t+1$.
		\ENDWHILE
	\end{algorithmic}
\end{figure} 
\section{Online Projected Gradient Descent} \label{section:adagrad} 
We use online sub-gradient descent with projection framework  as shown in \figurename \ref{algorithm:gradient_descent} for decision update. Together with the next theorem for the noiseless feedback scenario, this framework becomes the building block for the minimax optimal methods we derive in the noisy feedback sections to follow.

\begin{theorem} \label{theorem:gradient_descent_regret}
	If we use \figurename \ref{algorithm:gradient_descent} with a nonincreasing positive real $\eta_t$ sequence, we can guarantee the nonnegative regret bound
	\begin{displaymath}
		\regret \leq \frac{D^2}{2\eta_T} + \sum_{t=1}^T \frac{\eta_t}{2} \|\loss_t\|^2,
	\end{displaymath}
	where $D = \max_{\decision,\altdec \in \Set} \|\decision-\altdec\|$ is the diameter of $\Set$, the feasible convex decision set.
\end{theorem}

\begin{proof}
	Define $\altdec_{t+1} = \decision_t - \eta_t \loss_t$ with $\decision_{t+1} = \proj{\Set}{\altdec_{t+1}}$. 
	After some algebra following the substitution of $\loss_t$ in \eqref{eq:linearized_regret}, 
	\begin{equation*}
		\regret \leq \sum_{t=1}^T \frac{1}{2\eta_t} (\|\decision_t - \decision^*\|^2 - \|\decision_{t+1} - \decision^*\|^2) + \frac{\eta_t}{2} \|\loss_t\|^2,
	\end{equation*}
	since $\|\decision_{t+1} - \decision^*\| \leq \|\altdec_{t+1} - \decision^*\|$ for $\decision^* \in \Set$ \cite{Bubeck}. 
	After we regroup for each $\|\decision_t - \decision^*\|$ and upper-bound them with the set diameter (since $\eta_t$ is nonincreasing), we obtain a telescoping sum, which, after simplification, concludes the proof. 
\end{proof}
Theorem \ref{theorem:gradient_descent_regret} tells us 
that we need to tune the step sizes to optimize the regret (performance). The following corollary gives us the step sizes for the minimax optimal regret.

\begin{corollary} \label{corollary:gradient_descent_adaptive_eta}
	If we use $\eta_t = \sqrt{1/2} D G_t^{-1}$, we obtain
	\begin{displaymath}
		\regret
		\leq \sqrt{2} D G_T
		= \sqrt{2} D \sqrt{\sum_{t=1}^{T} \|\loss_t\|^2},
	\end{displaymath}
	where $G_t 
	\define \sqrt{G_{t-1}^2 + \|g_{t}\|^2}
	= \sqrt{\sum_{\tau=1}^t \|g_\tau\|^2}$ with $G_0 = 0$.
\end{corollary}
\begin{proof} 
	$G_t$ is a nondecreasing nonnegative sequence. Until $G_t > 0$ for some $t=t_0$, we incur $0$ regret. For $t \geq t_0$, $\eta_t$ becomes a nonincreasing sequence and we can build upon Theorem \ref{theorem:gradient_descent_regret}. 
	Then, $\|\loss_t\|^2 = G_t^2 - G_{t-1}^2$ where $t \geq 1$. Combined with Theorem \ref{theorem:gradient_descent_regret} and "difference of two squares",
	\begin{equation*}
		\regret \leq 
		\frac{D^2}{2\eta_T} + \sum_{t=t_0}^T \frac{\eta_t}{2} (G_t - G_{t-1})(G_t + G_{t-1}).
	\end{equation*}
	As $\eta_t$'s are positive and $G_t$'s are nondecreasing, we upper-bound by replacing $(G_t + G_{t-1})$ with $2G_t$. Then, we use $\eta_t = \sqrt{1/2} D G_t^{-1}$ and obtain a telescoping sum, which gives the bound after negative values are bounded with $0$.
\end{proof}
The upper bound (guarantee) in Corollary \ref{corollary:gradient_descent_adaptive_eta} matches the minimax regret lower bounds given in \cite[Theorem 5]{Abernethy} and \cite[Theorem 5]{Orabona} within a constant positive multiplier, meaning, it cannot be improved further order-wise.

\section{Performance for Noisy Gradients} \label{section:unbiasedgrad}
In this section, we consider the case where $\loss_t$ are observed with noise. Assuming we observe an unbiased estimator $\observe_t$ instead, we investigate how the regret of our algorithm is affected. We improve the results in 
\cite[Section 3]{Flaxman} by achieving temporally-adaptive minimax-optimal bounds for the first and second moments of the cumulative regret under noisy sub-gradient observations. 

We start by denoting the noise at $t$ as
\begin{equation} \label{eq:noise_t}
\noise_t \define \observe_t - \loss_t.
\end{equation}
Using \eqref{eq:noise_t}, we substitute for $\loss_t$ in \eqref{eq:linearized_regret}, which gives
\begin{equation} \label{eq:noisy_regret}
\regret 
\leq \sum_{t=1}^T \observe_t^T(\decision_t - \decision^*) - \sum_{t=1}^T \gamma_t^T (\decision_t - \decision^*).
\end{equation}

Next, we generate the regret guarantees in terms of the first and second moments in the following two subsections.

\subsection{First Moment (Expected Regret) Guarantee} \label{subsection:E[R_T]_unbiased}
\begin{lemma} \label{lemma:tildebound}
	If we run the algorithm in \figurename \ref{algorithm:gradient_descent} with $\eta_t$ as in Corollary \ref{corollary:gradient_descent_adaptive_eta} after substituting $\loss_t$ with $\observe_t$, we obtain
	\begin{equation*}
	\sum_{t=1}^T \observe_t^T (\decision_t - \decision^*) \leq \sqrt{2} D \sqrt{\sum_{t=1}^T \|\observe_t\|^2}.
	\end{equation*}
\end{lemma}
\begin{proof}
	It follows after substituting for $\loss_t$ with $\observe_t$ and considering that Corollary \ref{corollary:gradient_descent_adaptive_eta} results from \eqref{eq:linearized_regret} via Theorem \ref{theorem:gradient_descent_regret}.
\end{proof}

\begin{corollary} \label{corollary:expected_regret_unbiased}
	If we run the algorithm in \figurename \ref{algorithm:gradient_descent} with optimal adaptive $\eta_t$ for the noisy sub-gradient sequence $\{\observe_t\}_{t=1}^T$, we have the following expected regret guarantee:
	\begin{displaymath}
		\E{1:T}{\regret} 
		\leq \sqrt{2} D \;\E{1:T}{\sqrt{\sum_{t=1}^{T} \|\observe_t\|^2}} 
	\end{displaymath}
	where $\E{1:T}{\cdot}$ is the overall expectation 
	with the expected value of $\gamma_t$ conditioned to past is zero, i.e. $\E{t}{\gamma_t} = 0$. 
\end{corollary}
\begin{proof}
	$\eta_t$ is selected in accordance with Corollary \ref{corollary:gradient_descent_adaptive_eta}. Combining \eqref{eq:noisy_regret} with Lemma \ref{lemma:tildebound} and applying 
	$\E{1:T}{\cdot}$ 
	results in
	\begin{equation*}
		\E{1:T}{\regret} \leq \E{1:T}{\sqrt{2} D \sqrt{\sum_{t=1}^T \|\observe_t\|^2} - \sum_{t=1}^T \gamma_t^T (\decision_t - \decision^*)}.
	\end{equation*}
	The component
	$\sum_{t=1}^{T} \gamma_t^T (\decision_t - \decision^*)$ vanishes inside $\E{1:T}{\cdot}$ since the conditional expectation $\E{t}{\gamma_t} = 0$ results in $\E{t}{\gamma_t^T (\decision_t - \decision^*)} = 0$ for the following reasons. Firstly, the conditioning 
	includes our decision $\decision_t$. Secondly, $\decision^*$ can be treated as being defined at the start $t=1$, hence independent of the random events from $1\leq t\leq T$, by noticing that generation of the performance bounds 
	only assume $\decision^* \in \Set$, and are valid for all 
	such $\decision^*$. 
	This concludes the proof for this corollary. 
\end{proof}

\subsection{Second Moment (Expected Squared Regret) Guarantee} 
To bound the second moment, we incorporate 
Lemma \ref{lemma:tildebound} into \eqref{eq:noisy_regret},
and square both sides. Since $(a+b)^2 \leq 2a^2 + 2b^2$, we get
\begin{equation} \label{eq:R_T^2}
R_T^2 \leq 4D^2 \sum_{t=1}^T \|\observe_t\|^2 + 2\left(\sum_{t=1}^T \gamma_t^T (\decision_t - \decision^*)\right)^2.
\end{equation} 

\begin{lemma} \label{lemma:expected_squared_noise_regret}
	Given the overall expectation $\E{1:T}{\cdot}$,
	\begin{displaymath}
	\E{1:T}{\left(\sum_{t=1}^T \gamma_t^T (\decision_t - \decision^*)\right)^2} \leq D^2 \;\E{1:T}{\sum_{t=1}^T \sigma_t^2},
	\end{displaymath}
	where $\sigma_t^2 = \E{t}{\|\gamma_t\|^2}$, which is an aggregated noise variance measure for the unbiased estimator $\observe_t$ according to \eqref{eq:noise_t}.
\end{lemma}
\begin{proof}
	Denote $A_T = \sum_{t=1}^T (\decision_t - \decision^*)^T \gamma_t \gamma_t^T (\decision_t - \decision^*)$, and $B_t = \sum_{1\leq \tau < t} (\decision_\tau-\decision^*)^T \gamma_\tau \gamma_t^T (\decision_t - \decision^*)$. Then,
	\begin{equation*}
	\left(\sum_{t=1}^T \gamma_t^T (\decision_t - \decision^*)\right)^2 
	\leq  
	A_T
	+ 
	2\sum_{t=1}^{T} B_t, 
	\end{equation*}
	after we expand the square and reorganize right-hand side by grouping equivalent inner products ($a^T b = b^T a$).
	We take expectation of both sides. $B_t$ vanish since $\E{t}{\gamma_t}$ is $0$ and other terms are unaffected by this expectation. To upper bound $A_T$, we note that the only non-zero eigenvalue of $\gamma_t \gamma_t^T$ is $\|\gamma_t\|^2$ and, since $D^2 \geq \|\decision_t - \decision^*\|^2$, we can upper bound each summand in $A_T$ by $D^2\|\gamma_t\|^2$. Consequently, 
	after we note $\sigma_t^2 = \E{t}{\|\gamma_t\|^2}$, we obtain the lemma.
\end{proof}

\begin{corollary} \label{corollary:E(R_T^2)_unbiased}
	The second moment of the regret is bounded as
	\begin{align*}
	\E{1:T}{R_T^2} 
	&\leq D^2 \;\E{1:T}{\sum_{t=1}^{T} 4\|\loss_t\|^2 + 6\sigma_t^2} 
	\\& \leq 6 
		D^2 \;\E{1:T}{\sum_{t=1}^{T} \|\observe_t\|^2} 
	.
	\end{align*}
\end{corollary}
\begin{proof} 
	Note that $\E{t}{\|\observe_t\|^2} = \|\loss_t\|^2 + \sigma_t^2$. Combining with \eqref{eq:R_T^2} and Lemma \ref{lemma:expected_squared_noise_regret} gives the first inequality. We can further bound (and simplify) by 
	noting $4\|\loss_t\|^2 + 6\sigma_t^2 \leq 6\;\E{t}{\|\observe_t\|^2}$. 
\end{proof} 

\subsection{Minimax Optimality of the Regret Moment Guarantees}
In this subsection, we show the minimax optimality of these two upper bounds for the cumulative regret moments.
In a possibly adversarial setting, it has been shown that there exists a sequence of loss functions $\{f_t(\cdot)\}$ such that the regret we incur, i.e. $R_T^*$, is lower bounded as \cite{Orabona}:
\begin{equation} \label{eq:regret_lower}
	\regret^* \geq \frac{D}{2\sqrt{2}} \sum_{t=1}^{T} L_t^2,
\end{equation}
where $L_t$ is the least upper-bound on the norm of sub-gradient $\loss_t$, i.e. Lipschitz-continuity constant at time $t$.
After we observe $\loss_t$, we can 
guarantee that $L_t \geq \|\loss_t\|$. Hence, we can further lower bound the regret by replacing $L_t$ with $\|\loss_t\|$. Now, consider the case that we observe the noisy sub-gradient $\observe_t$ with no knowledge about the nature of noise $\gamma_t$. Then, it may very well be that $\|\observe_t\| = \|\loss_t\|$, i.e. there is no observation noise. Then, from the observer's perspective, we cannot incur regret lower than $\Omega\left(\sqrt{\sum_{t=1}^T \|\observe_t\|^2}\right)$ in an adversarial setting, which makes our guarantees 
in Corollaries \ref{corollary:expected_regret_unbiased} and \ref{corollary:E(R_T^2)_unbiased}, respectively, minimax optimal. 

\begin{remark} 
The removal of noise corresponds to the adversarial scenario when the worst-case regret lower-bound is generated in a min-max fashion (minimization by the user 
and maximization by 
the environment, i.e. adversary). 
We note 
the inherent trade-off between the 
aggregate noise variance ($\E{t}{\|\noise_t\|^2}$) and the sub-gradient norm square ($\|\loss_t\|^2$) when the squared norm of noisy sub-gradient ($\|\observe_t\|^2$) is the reference point. Hence,  
the removal of a stochastic noise results in the adversarial setting 
where an adversary may choose the sub-gradients to hinder the learning. Taking away from the
sub-gradient norms 
would weaken the said adversary.
\end{remark}

\section{Learning with Missing Gradients} \label{section:missing_gradient} 
In this section, we study an extreme noise scenario where the sub-gradients are not simply affected by an additive noise but instead their observability is
stochastic. Specifically, we study the setting where sub-gradients are probabilistically either observed or not (binary outcome). We continue to utilize the online sub-gradient descent framework 
by generating unbiased 
estimates to replace the true sub-gradients, which may again result from possibly adversarial functions. 

\subsection{Observations with Known Probabilities}
We consider the case where we can observe the sub-gradient $\loss_t$ with probability $\prob_t$ at each $t \geq 1$. 
When $\loss_t$ is observed, we also learn of $p_t$. We construct the unbiased estimator $\observe_t$ as
\begin{equation} \label{eq:observe}
\observe_t = 
\begin{cases}
\loss_t/\prob_t	&\text{ if observed,}\\
0	&\text{ otherwise.}
\end{cases}
\end{equation}
The observation noise $\noise_t$ is the same as before 
\begin{equation} \label{eq:noise}
\noise_t \define \observe_t - \loss_t.
\end{equation}
Next, we show the conditional expectations of $\observe_t$ and $\noise_t$.
\begin{lemma} \label{lemma:E_t}
	We have the following 
	conditional expectations:
	\begin{align*}
	&\E{t}{\|\observe_t\|^2} = \|\loss_t\|^2/\prob_t,
	\\&\E{t}{\|\noise_t\|^2} = \|\loss_t\|^2 (-1 + 1/\prob_t),
	\end{align*}
	where $\E{t}{\cdot}$ becomes the expectation over whether or not we make an observation at round $t$.
\end{lemma}
\begin{proof}
	Deriving from \eqref{eq:observe},
	\begin{equation*}
	\|\observe_t\|^2 =
	\begin{cases}
	\|\loss_t\|^2/\prob_t^2 &\text{if observed},\\
	0 &\text{otherwise}.
	\end{cases}
	\end{equation*}
	After taking expectation of the left-hand side, with regards to whether we observe at round $t$, we arrive at first expectation equality. The result for $\E{t}{\|\noise_t\|^2}$ comes after noting that $\E{t}{\|\observe_t\|^2} = \|\loss_t\|^2 + \E{t}{\|\noise_t\|^2}$. 
\end{proof}

\subsection{First and Second Moment Guarantees}
We now show the regret results for this observation failure setting. Since the observability of each sub-gradient has an underlying stochastic nature, our results can, again, only consist of guarantees on the regret moments. Similar to the noisy gradient analysis in Section \ref{section:unbiasedgrad}, we derive bounds for the first two moments, and, later, show both the expected value and variance of the regret are optimally bounded as a conclusion.
We start with the expected value guarantee.
\begin{corollary} \label{corollary:E[R_T]_known_p}
	If we run \figurename \ref{algorithm:gradient_descent} using the optimal adaptive step sizes $\{\eta_t\}_{t=1}^T$ for the unbiased sub-gradient substitutes $\{\observe_t\}_{t=1}^T$, as in \eqref{eq:observe}, we have the expected regret guarantee:
	\begin{align*}
		\E{1:T}{\regret} 
		&\leq \sqrt{2} D \;\E{1:T}{\sqrt{\sum_{t=1}^{T} \|\observe_t\|^2}}
		\\&\leq \sqrt{2} D \sqrt{\E{1:T}{\sum_{t=1}^{T} \frac{\|\loss_t\|^2}{\prob_t}}}.
	\end{align*}
\end{corollary}
\begin{proof}
	The first inequality is a direct result of Corollary \ref{corollary:expected_regret_unbiased} since $\observe_t$ from \eqref{eq:observe} is an unbiased estimator. For the second inequality, we first upper bound by moving $\E{1:T}{\cdot}$ into the square root as $\E{}{X} \leq \sqrt{\E{}{X^2}}$. 
	Then, inside this new expectation, we can further apply conditional expectations $\E{t}{\cdot}$ using Lemma \ref{lemma:E_t} and generate an equivalent bound.
\end{proof}
Next, we present the second moment guarantee.
\begin{corollary} \label{corollary:E[R_T^2]_known_p}
	Under the setting of stochastically failing observations (same with Corollary \ref{corollary:E[R_T]_known_p}) the corresponding second moment guarantee on $\E{1:T}{R_T^2}$ is
	\begin{align*}
	\E{1:T}{R_T^2} 
	&\leq D^2 \;\E{1:T}{\sum_{t=1}^{T} \|\loss_t\|^2\frac{(6-2\prob_t)}{\prob_t}} 
	\\& \leq 
	6D^2 \;\E{1:T}{\sum_{t=1}^{T} \|\observe_t\|^2} 
	.
	\end{align*}
\end{corollary}
\begin{proof}
	This is a direct consequence of incorporating Lemma \ref{lemma:E_t} into 
	the guarantee for unbiased estimators in 
	Corollary \ref{corollary:E(R_T^2)_unbiased}.
\end{proof}

\subsection{Minimax Regret Lower-Bound under Observation Failures}

In this subsection, we lower bound the worst-case (also called adversarial or minimax) regret, guaranteed when we make a total of $K$ observations at times $\{t_1,\ldots,t_K\}$ and the sub-gradients we observe have euclidean norms of $\{\|g_{t_1}\|,\ldots,\|g_{t_K}\|\}$. Later on, by showing that this lower bound is equivalent (up to a constant multiplier and additive, i.e. big-$O$ notation) to the regret moment guarantees we derived previously, we achieve the so-called minimax optimality. The importance of this result lies in the claim that there exist adversarial scenarios where the sequence of functions -possibly decided following our decisions- are such that we cannot do better than what we already guarantee (order-wise). 

\begin{lemma} \label{lemma:regret_lower_missing}
	When we make a total of $K$ sub-gradient observation at times $\{t_1,\ldots,t_K\}$, there exists a sequence of convex functions forcing any causal algorithm to incur the minimax total regret $\regretK^*$, lower bounded as
	\begin{displaymath}
		\regretK^* \geq \frac{D}{2\sqrt{2}} \sqrt{\sum_{k=1}^{K} L_k^2(t_k - t_{k-1})^2},
	\end{displaymath}
	where $D$ is the diameter of 
	$\Set$, $L_k$ is at least $\|g_{t_k}\|$ and $t_0 = 0$.
\end{lemma}
\begin{proof}
	We build upon the adaptive lower bound theorems by Abernethy and Orabona in \cite{Abernethy,Orabona}, respectively. We define the total regret, incurred after $K^{th}$ observation at time $t_K$, as
	\begin{equation} \label{eq:R_K}
		\regretK \define \sum_{t=1}^{t_K} f_t(x_t) - \min_{x \in \Set} \sum_{t=1}^{t_K} f_t(x),
	\end{equation}
	where $x_t \in \Set$, the 
	decision set, for all $t$.
	We define the sub-gradient space at each time $t \geq 1$ as 
	${G_t = \{\loss_t: \|\loss_t\| \leq L_t\}}$
	for an unknown bound $L_t$. Let us also define the set of Lipschitz-continuous convex functions to which  $f_t(\cdot)$ belongs: 
	\begin{equation} \label{eq:F_t}
		F_t = \{\text{convex }f_t| x_t \in \Set, f_t(x_t) \in \reals, \partial f_t(x_t) \subseteq G_t\},
	\end{equation}
	where $\partial f_t(x_t)$ is the set of sub-gradients when $f_t(\cdot)$ is evaluated at $x_t$ and $\Set$ is the common feasible decision set for all $t$. By only observing a sub-gradient at time $t$, we can restrict the function $f_t(\cdot)$ to the set $F_t$ for now.	
	
	Then, after observing the sub-gradients a total of $K$ times at $\{t_1,\ldots,t_K\}$, any causal algorithm may incur the regret $\regretK^*$ against an adversarial setting such that
	\begin{equation} \label{eq:R_K^*_minmax}
		\regretK^* = 
		\left\{
			\min_{\{x_t\}_{t\in J_k}: x_t \in \Set} \quad \max_{\{f_t(\cdot)\}_{t\in J_k}: f_t(\cdot) \in F_t}
		\right\}_{k=1}^K
		\regretK,
	\end{equation}
	where $\regretK$ is as defined in \eqref{eq:R_K}, $J_k = \{t_{k-1}+1,\ldots,t_k\}$ is the set of times between successive observations with $t_0 = 0$ for all $1\leq k\leq K$, and $\{\min_{A_k} \max_{B_k}\}_{k=1}^K$ is a short-hand notation for $\min_{A_1} \max_{B_1} \cdots \min_{A_K} \max_{B_K}$. 
	
	The minimax regret $\regretK^*$ is not computed in an alternation between $\min$ and $\max$ every round $t$, as in \cite{Abernethy} and \cite{Orabona}, but every observation $k$. The reason is that 
	any causal algorithm effectively decides $x_t$ for every $t_{k-1} < t \leq t_k$ following the last observation at time $t_{k-1}$ without further knowledge about $f_t(\cdot)$ until $t=t_k$. Consequently, to generate the minimum regret guarantee 
	possible 
	under an adversarial worst-case scenario, i.e. $\regretK^*$, the set of decisions $\{x_t\}_{t\in J_k}$ are selected until the observation at time $t=t_k$ without further feedback to minimize the worst-case regret resulting from the set of adversarial functions $\{f_t(\cdot)\}_{t\in J_t}$ for every $k$.
	
	Since the diameter of $\Set$ is $D$, identify the pair of points $u,v\in\Set$ such that $\|u-v\|=D$. Then, we can lower bound $\regretK^*$ by replacing $-\min_{x \in \Set}$ in $\regretK$ from \eqref{eq:R_K} with $-\min_{x \in \{u,v\}}$. To further lower bound $\regretK^*$, we also restrict the search spaces of $\max$ operations to $f_t(\cdot) = \langle Z_k L_t g, \cdot \rangle$ for $t_{k-1} < t \leq t_k$ where $Z_k \in \{-1,+1\}$ and $g = (u-v)/\|u-v\|$. This restriction is valid as $\partial f_t(x_t) \subseteq G_t$ holds and $f_t{\cdot} \in F_t$. 
	Then, replacing each $\max$ operations 
	with expectations over the unbiased random selection of $Z_k \in \{-1,+1\}$, we get
	\begin{align*}
		\regretK^* \geq 
			\cdots
			\min_{\{x_t\}_{t\in J_k}: x_t \in \Set}
			\mathbb{E}_{Z_k}
			\cdots
			\sum_{k=1}^{K} Z_k \sum_{t=t_{k-1}+1}^{t_k} \langle L_t g, x_t 
			-u^*\rangle, 
	\end{align*}
	where $u^* = \argmin{x\in\{u,v\}} \sum_{k=1}^{K} Z_k \sum_{t=t_{k-1}+1}^{t_k} \langle L_t g, x \rangle$.
	As $Z_k$ is randomly selected after $x_t$ is effectively decided for any causal algorithm for $t_{k-1} < t \leq t_k$, 
	$Z_k \langle L_t g, x_t \rangle$ vanishes inside the expectations $\mathbb{E}_{Z_k}$. Moreover, as 
	$\min_{\{x_t\}_{t\in J_k}: x_t \in \Set}$ do not affect 
	$Z_k \langle L_t g, u^* \rangle$, what remains is 
	\begin{equation*}
		\regretK^* \geq
			\E{Z_1:Z_K}{
				\max_{x\in\{u,v\}} \sum_{k=1}^{K} Z_k \sum_{t=t_{k-1}+1}^{t_k} \langle L_t g, x \rangle
			}
	\end{equation*}
	after collecting all $\mathbb{E}_{Z_k}$ under a single notation $\mathbb{E}_{Z_1:Z_K}$, plugging back the equivalence of $u^*$ and replacing $-\min_{x\in\{u,v\}}$ with $\max_{x\in\{u,v\}}$. After we substitute for the $\max$ operation 
	via $\max(A,B) = (A+B)/2 + |A-B|/2$, the 
	$"(A+B)/2"$ part vanishes in expectation. Then, since $g = (u-v)/\|u-v\|$ and, thus, $\langle g, u-v \rangle = \|u-v\| = D$, we are left with
	\begin{equation*}
		\regretK^* \geq
		\frac{D}{2} \; \E{Z_1:Z_K}{
			\left|
				\sum_{k=1}^{K} Z_k \sum_{t=t_{k-1}+1}^{t_k} L_t
			\right|
		}.
	\end{equation*}
	For the adversarial scenario, we can assume $L_t \geq \|\loss_{t_k}\|$ for $t_{k-1} < t \leq t_k$, since only then can the environment 
	produce a sub-gradient with norm $\|\loss_{t_k}\|$ at the time of $k^{th}$ observation, which could have possibly been before $t_k$. 
	Finally, 
	after applying Khinchin's inequality \cite{Cesa-Bianchi:2006}, 
	and using $L_t \geq \|g_{t_k}\|$, 
	we arrive at the lemma. 
\end{proof}

\subsection{Optimality of Guarantees under Observation Failures} \label{subsection:moment_optimality_observe_fail} 
We note that, up to now, we showed minimax optimality of the algorithm in \figurename \ref{algorithm:gradient_descent} for two settings, where we observe the true sub-gradient or an unbiased estimator, respectively. The minimax optimality of the true sub-gradient observation case was for the deterministic regret guarantee, while, for the unbiased estimator observation, we only provide guarantees on the first and second moments of the regret. Furthermore, the minimax optimality of the guarantees for unbiased estimators relied heavily on the fact that 
without further information on the observation noise (besides it having zero mean) 
the environment can select the noise in an adversarial manner by setting the variances per coordinate to $0$. 
However, in the observation failure setting, we also know the noise type and the corresponding aggregate variance, calculated in accordance with Lemma \ref{lemma:E_t}. Consequently, we next show the optimality for any probability to observe $p_{t_k}$ accompanying the $k^{th}$ sub-gradient observation $g_{t_k}$ as follows. We consider the restricted scenario where the adversarial setting selects the observation probability $p_{t_k}$ without a priori knowledge of observing time $t_k$, which further strengthens our argument. 
\begin{theorem}
	The guarantees given for the first and second moments of the regret (Corollaries \ref{corollary:E[R_T]_known_p} and \ref{corollary:E[R_T^2]_known_p}, respectively) are minimax optimal in accordance with Lemma \ref{lemma:regret_lower_missing} such that 
	\begin{equation*}
		\E{1:T}{(R_T)^\theta} 
		\text{ is }
		O\left(
		\E{1:T}{
			(\regretK^*)^\theta
		}
		\right)
		\text{ for }
		\theta \in \{1,2\}
	\end{equation*}
	where big-$O$ notation relates to $T$ growing without a bound.
\end{theorem}
\begin{proof}
	Consider the scenario 
	with predefined infinite length chains of observation probabilities $\{p_k\}_{k=1}^\infty$ and Lipschitz constants $\{L_k\}_{k=1}^\infty$. The adversarial setting starts with $p_1$ and $L_1$, and moves along the chains as we make observations, 
	i.e. the probability of making the $k^{th}$ observation ($p_{t_k}$) is $p_{k}$ and the observed sub-gradient is so that $\|g_{t_k}\| \leq L_k$. This is a valid 
	scenario for the $\max$ operators in \eqref{eq:R_K^*_minmax} as the algorithm only receives (utilizes) 
	$p_{t_k}$'s and $g_{t_k}$'s at random times $t_k$.
	
	Now, since $\observe_{t_k} = \loss_{t_k}/p_{k}$, 
	Corollaries \ref{corollary:E[R_T]_known_p} and \ref{corollary:E[R_T^2]_known_p} give 
	\begin{align*}
		&\E{1:T}{R_T} \leq \sqrt{2} D\; \E{K}{\E{1:T|K}{\sqrt{\sum_{k=1}^{K} \frac{\|g_{t_k}\|^2}{p_{k}^2}}}}, 
		\\&\E{1:T}{R_T^2} \leq 6 D^2\; \E{K}{\E{1:T|K}{\sum_{k=1}^{K} \frac{\|g_{t_k}\|^2}{p_{k}^2}}},
	\end{align*}
	after the decomposition of 
	$\E{1:T}{\cdot}$ into $\E{K}{\E{1:T|K}{\cdot}}$ where 
	$\E{1:T|K}{\cdot}$ conditions on making 
	$K$ observations, and replacing $\|\observe_t\|^2$ with $0$ when no observation is made. 
	Treating these guarantees as scaled moments of an 
	estimate $\widehat{\regret}$, 
	we get 
	\begin{equation} \label{eq:regret_estimate}
		\widehat{\regret} \define \sqrt{2}D\sqrt{\sum_{k=1}^{K} \frac{\|g_{t_k}\|^2} 
		{p_k^2}
		}.
	\end{equation}
	Next, we compare $\widehat{\regret}$ with $\regretK^*$ asymptotically ($T\rightarrow\infty$), which suffices for the big-$O$ claims in the theorem. For that, 
	we multiply both sides of Lemma \ref{lemma:regret_lower_missing} with $\sqrt{\sum_{k=1}^{K} L_k^2/p_k^2}$, which -in this scenario- is deterministic conditioned on $K$. Following that, we further lower bound as follows:
	\begin{align*}
		\regretK^* \sqrt{\sum_{k=1}^{K} \frac{L_k^2}{p_k^2}} 
		&\geq \frac{D}{2\sqrt{2}} \sqrt{
					\sum_{k,i=1}^{K} \frac{L_k^2}{p_k^2} L_i^2 (t_i - t_{i-1})^2
				}
		\\
		&\geq \frac{D}{2\sqrt{2}} \sqrt{
					\sum_{k,i=1}^{K} L_k^2 L_i^2\frac{(t_k - t_{k-1})(t_i - t_{i-1})}{p_k p_i}
				}
		\\
		&\geq \frac{D}{2\sqrt{2}} \sum_{k=1}^{K} L_k^2\frac{t_k - t_{k-1}}{p_k},
	\end{align*}
	where, in the first inequality, the right-hand side comes after the distribution of initial multiplication, the second inequality is achieved via the application of $a^2 + b^2 \geq 2ab$ and the final expression comes after realizing the square root operation. We apply the expectation $\E{1:T|K}{\cdot}$ to both sides with $T\rightarrow\infty$. Since $L_k$ and $p_k$ are deterministic when conditioned on $K$, and mean of the geometrically distributed $(t_k - t_{k-1})$ is calculated as $\E{1:T|K}{t_k - t_{k-1}}_{T\rightarrow\infty} = 1/p_k$, we obtain
	\begin{equation*}
		\E{1:T|K}{\regretK^*}_{T\rightarrow\infty} \geq \frac{D}{2\sqrt{2}} \sqrt{\sum_{k=1}^{K} \frac{L_k^2}{p_k^2}} \geq \frac{D}{2\sqrt{2}} \sqrt{\sum_{k=1}^{K} \frac{\|g_{t_k}\|^2}{p_k^2}},
	\end{equation*}
	since $L_k \geq \|g_{t_k}\|^2$. Applying $\E{1:T|K}{\cdot}$ once more achieves
	\begin{equation*}
		\E{1:T|K}{\regretK^*}_{T\rightarrow\infty} 
		\geq \frac{1}{4}\; \E{1:T|K}{\widehat{\regret}}_{T\rightarrow\infty},
	\end{equation*}
	as $\E{1:T|K}{\E{1:T|K}{\regretK^*}}_{T\rightarrow\infty} = \E{1:T|K}{\regretK^*}_{T\rightarrow\infty}$ and where we have substituted in $\widehat{\regret}$ using \eqref{eq:regret_estimate}. After taking expectation over $K$, i.e. $\E{K}{\cdot}$, we finally show minimax optimality in expectation with
	\begin{equation*}
		\E{1:T}{\regretK^*}_{T\rightarrow\infty} \geq \frac{1}{4} \; \E{1:T}{\widehat{\regret}}_{T\rightarrow\infty},
	\end{equation*}
	which shows that our expected regret guarantee $\E{1:T}{\widehat{\regret}}$ is asymptotically ($T\rightarrow\infty$) at most $4$ times the expectation of worst-case (adversarial) regret and implies the big-$O$ first moment guarantee in the theorem.
	
	The proof of second moment guarantee is similar and more straightforward. This time, we compare $(\regretK^*)^2$ and $(\widehat{\regret})^2$ asymptotically ($T\rightarrow\infty$). Similarly, we apply expectations, first $\E{1:T|K}{\cdot}$, then $\E{K}{\cdot}$. An inequality is achieved after $\E{1:T|K}{\cdot}$, where the second moment of the geometrically distributed $(t_k - t_{k-1})$ is calculated as $\E{1:T|K}{(t_k - t_{k-1})^2}_{T\rightarrow\infty} = (2-p_k)/p_k^2 \geq 1/p_k^2$. Again, 
	applying $\E{K}{\cdot}$ to both sides gives the minimax inequality with implication of the big-$O$ result. 
\end{proof}

\begin{remark}
	If the observation probability is fixed, i.e. $\prob_t = \prob$, then \figurename \ref{algorithm:gradient_descent} does not need to know $p$ in the decision update when using the unbiased estimator $\observe_t$, since the optimal adaptive step size $\eta_t$ also includes $p$ and they cancel each other when the update step $-\eta_t \observe_t$ is computed. Hence, it is sufficient to use $\observe_t = g_t$ if observed and $\observe_t=0$ otherwise when $\prob_t = \prob$.
\end{remark}

\section{Observation Probability Generated by a Prior} \label{section:prior} 
In the previous section, we have introduced a minimax algorithm that utilizes the observation probabilities of observed sub-gradients. However, such information may not necessarily be available to us. To this end, in this section, we modify the scenario such that we do not explicitly know the probabilities to observe sub-gradients at times $t$, i.e. $p_t$. Instead, we suppose that following each observation we make, a new probability to observe is randomly generated from a possibly non-stationary prior distribution. This probability is kept hidden and stationary, similar to the adversarial setting in our lower bound analysis for observation failures. Hence, our algorithm will not have access to the generated probability but only its prior distribution. We will use this approach as a building block for our algorithm in the next section, which does not need any information about the observation probabilities or their priors.

Denote the $k^{th}$ prior as $\pdf_k(x_k)$, which generates the next probability to observe following the $(k-1)^{th}$ observation, with the initial prior being $\pdf_1(x_1)$. These $\pdf_k(\cdot)$ can be arbitrarily 
dependent on any past data such as the previous priors, 
their realizations and 
the resulting observation times, i.e. $\{(\pdf_{l}(\cdot),x_l,t_l)\}_{l=1}^{k-1}$, 
and past target functions 
and our decisions, i.e. $\{(f_\tau(\cdot),\decision_\tau)\}_{\tau=1}^{t_{k-1}}$. 
With $t_0=0$, at round t such that $t_{k-1} < t \leq t_k$, while waiting for the $k^{th}$ observation at $t_k$, we require $\prob_t = \Pr(t_k = t| t_{k-1}, t_k \geq t)$, the conditional probability of making the $k^{th}$ observation at $t$, i.e. $t_k = t$, conditioned on the fact that we made the last observation at round $t_{k-1}$, which also implies 
$t_k \geq t$. 
Then, Bayes' rule gives 
\begin{equation} \label{eq:p_t}
\begin{aligned}
\prob_t 
&= \frac{\Pr(t_k = t, t_k \geq t | t_{k-1})}{\Pr(t_k \geq t|t_{k-1})}
= \frac{\Pr(t_k = t | t_{k-1})}{\Pr(t_k \geq t|t_{k-1})}
\\&= \frac
{\int_{0}^{1} \pdf_k(x) \; (1-x)^{t-t_{k-1}-1} x\; \deri x}
{\int_{0}^{1} \pdf_k(x) \; (1-x)^{t-t_{k-1}-1} \; \deri x},
\end{aligned}
\end{equation}
since under this observability setting, after the 
probability of making the $k^{th}$ observation is generated from a prior, 
the probability of making the next observation at round $t$, 
i.e. $(t-t_{k-1})$ rounds after the previous round of observation $t_{k-1}$, is geometrically distributed with failures to observe until $t$. 

We exemplify the behavior of $p_t$ derivation in \eqref{eq:p_t} by considering a mixture of arbitrary beta distributions and some probability mass function as the prior. We choose to investigate this example 
since the beta distribution family is the conjugate prior for geometric distributions and the mixtures, by nature, can substitute for arbitrary modalities. Hence, let us consider
\begin{equation} \label{eq:prior}
\begin{aligned}
\pdf_k(x) &= P^B_{k}(x) + P^M_{k}(x),
\\\text{with} \quad 
P^B_{k}(x) &= \sum_{b=1}^{N_k} \lambda^B_{k,b} \frac{x^{\alpha_{k,b}-1} (1-x)^{\beta_{k,b}-1}} {B(\alpha_{k,b},\beta_{k,b})}
\\ \text{and} \quad
P^M_{k}(x) &= \sum_{m=1}^{M_k} \lambda^M_{k,m} \delta(x-p_{k,m}),
\end{aligned}
\end{equation}
Here, $B(\alpha_{k,b},\beta_{k,b}) = \Gamma(\alpha_{k,b})\Gamma(\beta_{k,b})/\Gamma(\alpha_{k,b}+\beta_{k,b})$ are the corresponding normalizations for beta distributions such that $\int_{0}^{1} x^{\alpha_{k,b}-1} (1-x)^{\beta_{k,b}-1} \deri x = B(\alpha_{k,b},\beta_{k,b})$ for parameters $\alpha_{k,b},\beta_{k,b}>0$. Furthermore, $\delta(\cdot)$ is the Dirac delta function, which models the behavior of a probability mass such that $\int_{0}^{1} \delta(x-p) \deri x = 1$ for parameter $0<p\leq 1$ and, for all $x\neq p$, $\delta(x-p) = 0$. Lastly, $\lambda^B_{k,b}$ and $\lambda^M_{k,m}$ are the respective nonnegative mixture coefficients across beta distributions and probability masses such that $\sum_{b=1}^{N_k} \lambda^B_{k,b} + \sum_{m=1}^{M_k} \lambda^M_{k,m} = 1$. 

Before continuing further, we define the short-hands $p_{k,b}^t$, $Q_{k,b}(t)$ and $F_{k,m}(t)$ for 
$t_{k-1} \leq t < t_{k}$  
and $k\geq 1$ such that
\begin{align} \label{eq:prior_short_hand} 
	p_{k,b}^t &\define \frac{\alpha_{k,b}}{\alpha_{k,b}+\beta_{k,b}+(t-t_{k-1}-1)}, 
	\nonumber\\Q_{k,b}(t) &\define \frac{\Gamma(\alpha_{k,b}+\beta_{k,b})}{\Gamma(\beta_{k,b})}
	\frac{\Gamma(\beta_{k,b}+(t-t_{k-1}-1))}{\Gamma(\alpha_{k,b}+\beta_{k,b}+(t-t_{k-1}-1))}, 
	\nonumber\\F_{k,m}(t) &\define (1-p_{k,m})^{t-t_{k-1}-1}, 
\end{align}
which can be sequentially computed in an efficient manner as
\begin{align*}
	p_{k,b}^{t+1} &= p_{k,b}^{t} \frac{\alpha_{k,b}+\beta_{k,b}+(t-t_{k-1}-1)}{\alpha_{k,b}+\beta_{k,b}+(t-t_{k-1})},
	\\Q_{k,b}(t+1) &= Q_{k,b}(t) 
	\frac{\beta_{k,b}+(t-t_{k-1}-1)}{\alpha_{k,b}+\beta_{k,b}+(t-t_{k-1}-1)},
	\\F_{k,m}(t+1) &= F_{k,m}(t)(1-p_{k,m}).
\end{align*}
After combining \eqref{eq:p_t}, \eqref{eq:prior} and \eqref{eq:prior_short_hand} while using the functional relation $\Gamma(x+1)=x\Gamma(x)$ for a gamma function with real argument $x$ when decomposing ratios of beta functions $B(\cdot,\cdot)$,
\begin{equation*}
	\prob_t = \frac{
		\sum_{b=1}^{N_k} \lambda^B_{k,b} 
		\; Q_{k,b}(t) \;
		p_{k,b}^t 
		\enspace + \enspace
		\sum_{m=1}^{M_k} \lambda^M_{k,m} 
		\; F_{k,m}(t) \;
		p_{k,m} 
	}{
		\sum_{b=1}^{N_k} \lambda^B_{k,b} 
		Q_{k,b}(t)
		+ 
		\sum_{m=1}^{M_k} \lambda^M_{k,m} F_{k,m}(t)
	},
\end{equation*}
where, before the $k^{th}$ observation, the previously denoted $p_{k,b}^t$ and $p_{k,m}$ act as the opinions of $b^{th}$ beta distribution and $m^{th}$ probability mass in the $k^{th}$ prior mixture, respectively, for the conditional observation probability at time $t$. Also, their mixture weights renew as $\lambda^B_{k,b} Q_{k,b}(t)$ and $\lambda^M_{k,m} F_{k,m}(t)$.
\begin{remark}
	For a single beta distribution as the $k^{th}$ prior with parameters $\alpha_{k},\beta_{k} > 0$, observation probability becomes
	\begin{equation*}
		\prob_t = p_{k,1}^t
		= \alpha_k/\left(\alpha_k+\beta_k+(t-t_{k-1}-1)\right).
	\end{equation*} 
	For the uniform prior $\alpha_{k}=\beta_{k}=1$, 
	it is $ 
		\prob_t = (t-t_{k-1}+1)^{-1}. 
	$
\end{remark} 
\begin{remark}
	The maximum likelihood estimate of a fixed observation probability $p_t = p_k$ for $t_{k-1} < t \leq t_k$ is 
	\begin{equation*}
		\widehat{p_t} = (t_k - t_{k-1})^{-1},
	\end{equation*} 
	which resembles the case where, for an arbitrarily small $\epsilon > 0$, $\prob_t = (t - t_{k-1} + \epsilon)^{-1}$ due to the $k^{th}$ prior being a beta distribution with parameters $\alpha_k=1, \beta_k=\epsilon$. 
\end{remark}

\begin{remark}
	The computation of $\prob_t$ in accordance with \eqref{eq:p_t} is only actually realized when we make an observation at the future time $t=t_k$. Consequently, one may also need to incorporate both that time and the observed sub-gradient $\loss_t$ in the conditioning when calculating $\prob_t$ depending on the setting. As a reminder, the prior $\pdf_k(x_k)$ had already included the conditioning on past information 
	from $\tau \leq t_{k-1}$ as per definition. An example scenario where this additional conditioning could be required is such that when generating the random constant probability to observe, future sub-gradient sequence until the next observation may be restricted to different subspaces depending on the generated probability itself, providing additional information to employ following the observation and before the decision update.
\end{remark}

\subsection{Optimality of Probability from a Beta-Mass Mixture Prior}
The optimality analysis in Section \ref{subsection:moment_optimality_observe_fail} was for the scenario where we, as the optimizer, only knew about the observation probabilities $p_t$ whenever we observed a sub-gradient and had no further information on how they were generated. 
Knowing a prior without the explicit knowledge of $p_t$ differs from this scenario.
Therefore, we need fresh optimality claims for when we have stationary observation probabilities between successive observations each generated from some prior distribution. 
To show the optimality of regret guarantees corresponding to generating $p_t$ from the known prior considered in \eqref{eq:prior}, which is a mixture of beta distributions and probability masses, we first lower bound the said $p_t$ as follows.
\begin{lemma} \label{lemma:prior_p_t_lower}
	The probability of making the $k^{th}$ observation at time $t$ for $t_{k-1} < t \leq t_k$ for a given mixture prior $\pdf_k(\cdot)$, i.e. $\prob_t$, can be lower bounded such that
	\begin{equation*}
		\prob_t \geq \min
		\left(
			\min_{1\leq b\leq N_k} p_{k,b}^t, \min_{1\leq m\leq M_k} p_{k,m}
		\right),
	\end{equation*}
	where $p_{k,b}^t$ and $p_{k,m}$ are the conditional observation probabilities generated at time $t$ by the $b^{th}$ beta distribution and $m^{th}$ probability mass function in the $k^{th}$ mixture prior, respectively, in accordance with \eqref{eq:prior} and \eqref{eq:prior_short_hand}.
\end{lemma}
\begin{proof}
	As previously shown, $p_t$ is a convex combination of $p_{k,b}^t$ and $p_{k,m}$ for $1\leq b\leq N_k$ and $1\leq m\leq M_k$, with nonnegative weights $\lambda^B_{k,b} 
	Q_{k,b}(t)/Z_k(t)$ and $\lambda^M_{k,m}
	F_{k,m}(t)/Z_k(t)$, where
	$Z_k(t) = \sum_{b=1}^{N_k} \lambda^B_{k,b}Q_{k,b}(t)
			+ \sum_{m=1}^{M_k} \lambda^M_{k,m}F_{k,m}(t)$ is the normalization factor so that all weights sum to $1$.
	Consequently, $p_t$ is at least the minimum of those combined. 
\end{proof}

What Lemma \ref{lemma:prior_p_t_lower} shows us -in combination with \eqref{eq:prior_short_hand}- is that $(1/\prob_t)$ is $O(t_k-t_{k-1})$ for $t_{k-1} < t \leq t_k$ whenever the parameters of beta distributions and mass functions in the $k^{th}$ prior are such that $\alpha_{k,b}$ is $\Omega(1)$, $\beta_{k,b}$ is $O(t_k - t_{k-1})$ and $p_{k,m}$ is $\Omega((t_k - t_{k-1})^{-1})$ for $1\leq b\leq N_k$ and $1\leq m\leq M_k$ for each $1\leq k\leq K$, respectively. Most notably, these constraints hold when these parameters -which are determined before time $t_k$ is realized- belong to $\Theta(1)$, e.g. they are a priori determined and sequentially revealed as observations are made. Consequently, 
the first and second regret moment guarantees in Corollaries \ref{corollary:E[R_T]_known_p} and \ref{corollary:E[R_T^2]_known_p} match -up to a constant factor- with the respective moments of the adversarial regret lower bound in Lemma \ref{lemma:regret_lower_missing}. 

\begin{figure}[!t]
	\centering
	\caption[caption]{APGD.EP: Adaptive Projected Sub-Gradient Descent \\ (with Empirical Probability Estimates)}\label{algorithm:empirical}
	\begin{algorithmic}[1]
		\REQUIRE $\decision_1 \in \Set$, $\loss_t \in \subgrads{f_t}{\decision_t}$ for $t=t_k$ for some $k\geq 1$.
		\ENSURE $\decision_t \in \Set$ for $t \geq 2$.
		\STATE Initialize $t=t_1$ (wait until $1^{st}$ observation), $k=1$, $Z=1$.
		\STATE $J = \{t_1\}$: list of ascending occurred time-differences
		\STATE $C = \{1\}$: count of time difference occurrences
		\STATE $Z$ : denominator for $\widehat{p_t}$ computation
		\WHILE{$\loss_t = 0$}
		\STATE $k \leftarrow k+1$. \label{line:goto1}
		\STATE $Z = k$, $i=1$ (index to update for $J$ and $C$).
		\WHILE[wait until next observation]{$t<t_k$}
		\STATE $t \leftarrow t+1$.
		\IF[if not, will insert at the end]{$i \leq length(J)$}
		\IF{$t-t_{k-1} > J(i)$} 
		\STATE $Z \leftarrow Z - C(i)$.
		\STATE $i \leftarrow i+1$.
		\ENDIF
		\ENDIF
		\ENDWHILE 
		\IF{$i \leq length(J)$ \AND $t_k-t_{k-1} = J(i)$} 
		\STATE $C(i) \leftarrow C(i)+1$
		\ELSE[Using Linked List Insert]
		\STATE Insert $(t_k-t_{k-1})$ to $J$ so that $J(i) = t_k-t_{k-1}$.
		\STATE Insert $1$ at $i^{th}$ index of $C$.
		\ENDIF \label{line:exit1}
		\ENDWHILE
		\STATE Initialize $G = 0$, $D$: diameter of $\Set$.
		\WHILE{not terminated}
		\STATE Observe $\loss_t \in \subgrads{f_t}{\decision_t}$.
		\IF{$\loss_t \neq 0$}
		\STATE Compute $\widehat{p_t} = C(i) / Z$. (Implicit Common Prior)
		\STATE Set $\widetilde{\loss_t} = \loss_t / \widehat{p_t}$. (Unbiased Estimate)
		\STATE $G \leftarrow G + \|\widetilde{\loss_t}\|^2$. (Adaptive Normalizer)
		\STATE Decide $\eta_t = \sqrt{1/2} D/G$. 
		(Corollaries \ref{corollary:gradient_descent_adaptive_eta} and \ref{corollary:expected_regret_unbiased})
		\STATE Set $\decision_{t+1} = \proj{\Set}{\decision_t - \eta_t \widetilde{\loss_t}}$. (Update)
		\ENDIF 
		\STATE Update $k$, $t$ and $J$, $C$, and $Z$, $i$ similar to Lines \ref{line:goto1}-\ref{line:exit1}.
		\ENDWHILE
	\end{algorithmic}
\end{figure}
\section{Empirical Probability Estimation} 
\label{section:empirical}
In this final methodology section, we discuss how to proceed when neither the marginal probability of making an observation at time $t$, i.e. $\prob_t$, nor a prior for the stationary probability of making the $k^{th}$ observation, i.e. $\pdf_k(\cdot)$, are revealed. We build up on the previous prior-generated stationary probability scheme by assuming an unknown common prior $\pdf(\cdot)$ such that $\pdf_k(\cdot) = \pdf(\cdot)$ for all $k\geq 1$. We shall not impose any further restrictions upon the prior $\pdf(\cdot)$, i.e. it is not necessarily a mixture of beta distributions and mass functions like exemplified in the previous section. 
For a common prior, following \eqref{eq:p_t}, $\prob_t$ becomes such that
\begin{equation} \label{eq:p_t_uniform_prior}
	\prob_t 
	= \frac{\Pr(t_k=t|t_{k-1})}{\Pr(t_k\geq t|t_{k-1})} 
	= \frac{\Pr(t_k-t_{k-1}=t-t_{k-1})}{\Pr(t_k-t_{k-1}\geq t-t_{k-1})}, 
\end{equation}
where $t_{k-1}$ conditioning disappears due to $\pdf_k(\cdot) = \pdf(\cdot)$. Consequently, the time differences between successive observations, i.e. each $t_k - t_{k-1}$, become independent and identically distributed random variables. We can rewrite $\prob_t$ in \eqref{eq:p_t_uniform_prior} using the cumulative distribution function $F(\cdot)$ with $F(t-t_{k-1}) = \Pr(t_k - t_{k-1} \leq t-t_{k-1})$. Afterwards, we shall approximate $p_t$ by replacing $F(\cdot)$ with $F_t(\cdot)$ such that
\begin{equation} \label{eq:p_t_hat}
	\widehat{p_t} = 1 - \frac{1 - F_t(t-t_{k-1})}{1- F_t(t-1-t_{k-1})},
\end{equation}
where $F_t(\cdot)$ is the empirical distribution function, which converges uniformly to $F(\cdot)$ according to Glivenko-Cantelli theorem \cite{Vaart}. We present the pseudo-code in Algorithm \ref{algorithm:empirical}, which includes the efficient computation of $F_t(t-t_{k-1})$ using the differences $P_t(t-t_{k-1}) = F_t(t-t_{k-1}) - F_t(t-1-t_{k-1})$, i.e. the total number of times $t_{k'}-t_{k'-1}$ for any $1\leq k'\leq k$ has equaled $t-t_{k-1}$ divided by the number of observations $k$. 
Since $p_t$ is only utilized when an observation is made, we can assume that $t_k = t$ in the calculation of $p_t$. 
\begin{remark}
	Interestingly, this empirical
	approach can model not only constant probabilities independently generated from a common prior $\pdf(\cdot)$, but also any shift-invariant conditional probability conforming to \eqref{eq:p_t_uniform_prior}. Furthermore, the empirical distribution $F_t(\cdot)$ can be separately computed for a number of disjoint sub-gradient sets, i.e. $F_t^{G}(\cdot)$ for sub-gradient set $G$, 
	and 
	\eqref{eq:p_t_hat} can be computed by using the corresponding $F_t^G(\cdot)$ with 
	$\loss_t \in G$. 
	Moreover, with sufficient computation power, the collection of pairs $\left(G,F_t^{G}(\cdot)\right)$
	can be updated 
	with time 
	to preserve similar convergence rates between $F_t^{G}(\cdot)$.
\end{remark} 

\section{Experiments} \label{section:experiments}
In this section, we compare the performance of \figurename \ref{algorithm:empirical}, henceforth called 'APGD.EP', against various competitors derived from the online sub-gradient descent framework in \figurename \ref{algorithm:gradient_descent} under the setting of stochastic feedback failure for online convex optimization. 
These competitors either have knowledge regarding the nature of stochastic feedback failure and can guarantee convergence in a minimax optimal manner, or they employ a standard approach towards remedying their knowledge deficiency without such guarantees.
As opposed to these competitors, 'APGD.EP' (\figurename \ref{algorithm:empirical}), is an online sub-gradient descent algorithm with an intricate empirical approach to resolve the information deficiency by learning the underlying nature of the stochastic feedback failure. 
Our experiments
demonstrate that 'APGD.EP' can perform similar to the minimax optimal algorithms without the a priori knowledge they can access. Moreover, the experiments also show that 'APGD.EP' can perform similar to the competitors which are 
tuned for the further assumption of 
the target function sequence 
being either randomized or adversarial. 

\subsection{Competing Adaptive Descent Variants}
Six adaptive sub-gradient descent variants are tested under the labels 'Ignore', 'w/Prior', 'w/Known', 'GML', 'Uniform', 'APGD.EP'. They differ in their estimates $\observe_t$ for the sub-gradients 
and corresponding step sizes $\eta_t$. 

\textbf{Ignore:} 
It uses $\observe_t=\loss_t$ when available and, thus, 
ignores the likelihood to observe and naively updates the estimate whenever data is available.
It works well for the non-adversarial scenario when the function sequence is randomized. 

\textbf{w/Known:} 
It knows the conditional probability of observation $\prob_t$ and 
computes the unbiased estimate as $\observe_t=(1/\prob_t)\loss_t$ if observed. 
It is minimax optimal.

\textbf{w/Prior:} 
It does not know the actual $p_t$ but knows the priors that randomly generate the probabilities to observe after each observation, and computes $p_t$ and $\observe_t=(1/\prob_t)\loss_t$, if observed, accordingly. 
It is also minimax optimal.

\textbf{Uniform:} 
It knows neither $\prob_t$ nor its prior. As a remedy, it naively assumes a uniform prior 
following all observations and, as a result, computes $\observe_t = (t-t_{k-1}+1)\loss_t$ if observed where $t_{k-1}$ is the time of last observation, with $t_0 = 0$. 

\textbf{GML:} It is the greedy maximum likelihood approach such that it assumes a constant observation probability between successive observations like 'w/Prior'. Then, it estimate $p_t$ 
by maximizing the probability of ($t_k = t$ given $t_{k-1}$), which results in 
$\observe_t=(t-t_{k-1}) \loss_t$ if observed. 
Interestingly, it is optimal for the adversarial scenario in Lemma \ref{lemma:regret_lower_missing}. 

\textbf{APGD.EP:} It is the empirical approach in \figurename \ref{algorithm:empirical}. It knows neither $\prob_t$ nor the priors generating them. It derives an empirically convergent approximation of $p_t$ assuming each observation -and the failures leading to it- 
as IID events. 
Its assumption encapsulates the scenario of an unknown common prior distribution following all observations. 

Next three subsections detail the experiments.

\renewcommand{\figurename}{Fig.}
\begin{figure*}[!t]
	\centering
	\subfloat[]{
		\includegraphics[width=0.5\textwidth]{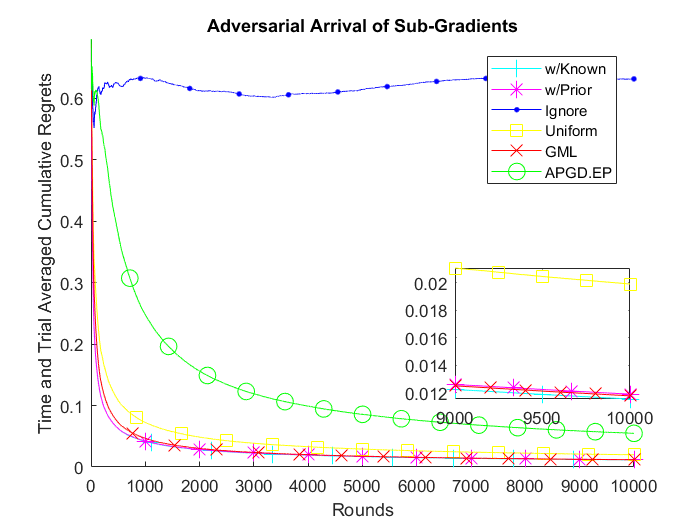}
		\label{figure:adversarial_cumulative}
	} 
	\subfloat[]{
		\includegraphics[width=0.5\textwidth]{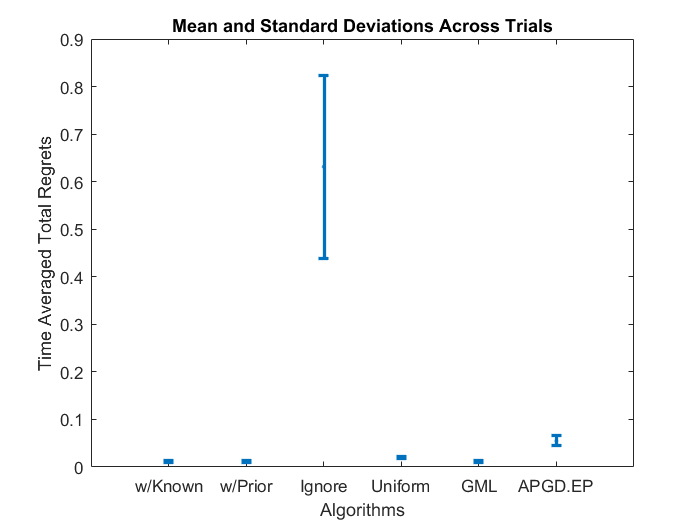}
		\label{figure:adversarial_mean_std}
	}
	\newline 
	\subfloat[]{
		\includegraphics[width=0.5\textwidth]{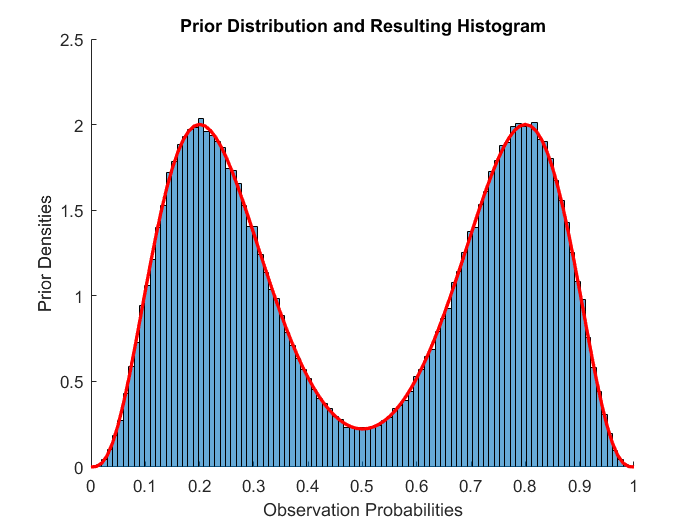} 
		\label{figure:prior}
	}
	\caption{(a) Time-averaged cumulative regret behavior of 6 algorithms for adversarial sub-gradients. (b) Mean and SD of the time-averaged total regrets across trials. (c) Prior distribution used for generating random constant observation probabilities between successive observations.} 
\end{figure*}

\subsection{Simulated Adversarial Sub-Gradient Arrival} \label{subsec:simulated_adversary}
In the adversarial scenario used in our lower bound analyses, the sub-gradient vector was sampled from a 1D line with a random direction. Nevertheless, we use 16 dimensional decision (and gradient) space for a more obvious sense of generality. We restrict the decision space to the unit 
radius hypersphere and the adversarial function sequence is composed of linear functions, which are 
arbitrarily chosen as the inner product of the decision with a scaled all-one vector. The scale is randomly selected as $0.25$ or $-0.25$ following each observation. 

A summary of time-averaged cumulative regrets is presented in \figurename \ref{figure:adversarial_mean_std}. Each error bar represents the mean and standard deviation of the time-averaged (over 10000 rounds) cumulative loss of an algorithm  over 50 trials. During the online procedure, following each observation, a new random probability to observe is generated from a bimodal common prior, a uniform mixture of beta distributions with parameters $\alpha_1=4, \beta_1=13$ and $\alpha_2=13, \beta_2=4$, respectively, as shown in \figurename \ref{figure:prior}.
\figurename \ref{figure:adversarial_cumulative}, which displays the evolution of average regret performance, with zooming to the final rounds, and \figurename \ref{figure:adversarial_mean_std} show 'APGD.EP' can perform similar to minimax optimal algorithms 'w/Known' and 'w/Prior' both of which require knowledge about the observation probabilities. 'APGD.EP' can also compete with 'Uniform' and 'GML' both 
of which have an unfair advantage in specifically 
adversarial scenarios as their updates correspond to (almost for 'Uniform' and exactly for 'GML') the actual sub-gradient accumulation since the previous observation. In essence, 'APGD.EP' succeeds since it can learn the underlying nature of stochastic feedback. Furthermore, as expected, 'Ignore' performs poorly in such a scenario as it ignores the time past since last observation. Thus, it is the only one whose average regret cannot converge to 0.

\begin{figure*}[!t]
	\centering
	\subfloat[]{
		\includegraphics[width=0.5\textwidth]{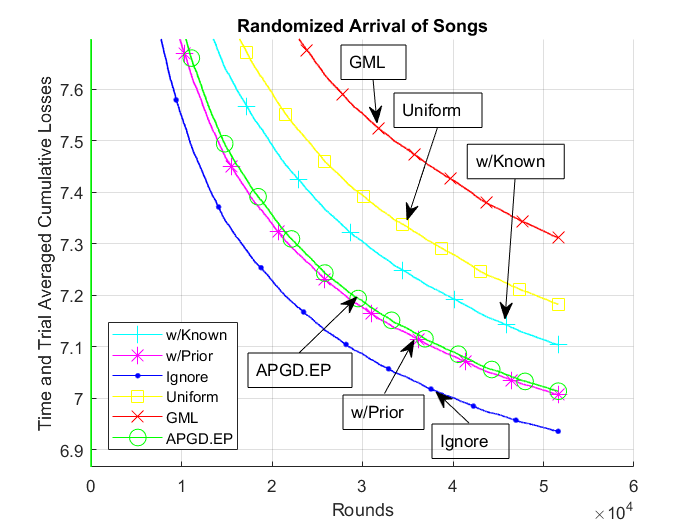}
		\label{figure:song_test_set_cumulative}
	} 
	\subfloat[]{
		\includegraphics[width=0.5\textwidth]{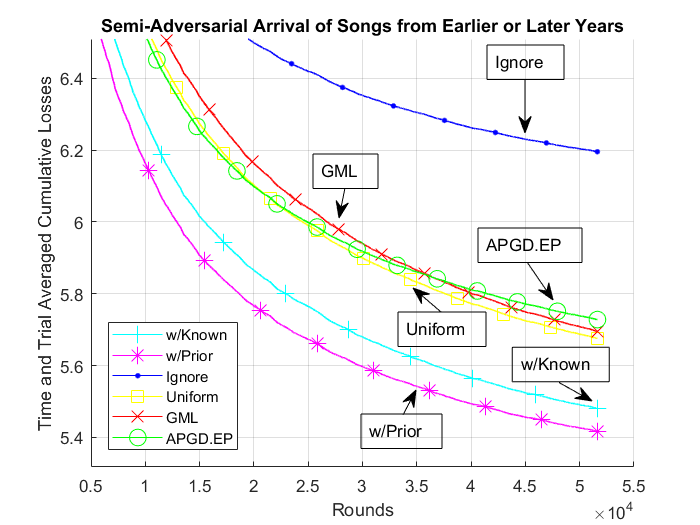}
		\label{figure:song_adversarial_cumulative}
	}
	\newline 
	\subfloat[]{
		\includegraphics[width=0.5\textwidth]{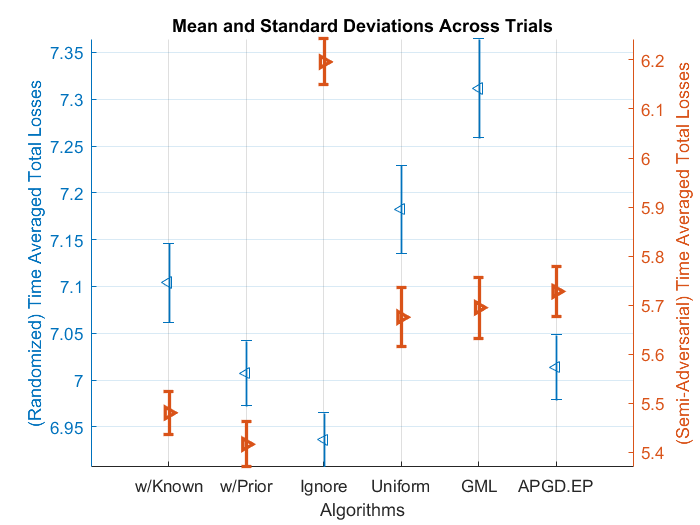}
		\label{figure:song_mean_std}
	}
	\caption{Time-averaged cumulative losses of 6 algorithms  for release year prediction from the Million Song Dataset for a standard and an adversarial scenarios: (a) Non-adversarial scenario with songs arriving randomly without pattern (b) Semi-adversarial scenario with songs between successive observations all from either earlier (before 2000) or later years (2000 and after). (c) Mean and SD of the total losses from 6 algorithms for 2 cases.}
\end{figure*}

\subsection{Linear Regression with Absolute Error for Year Prediction}
\label{subsec:linear_regression}
In this experiment, the aim is year prediction over a subset of the Million Song Dataset \cite{Bertin-Mahieux} obtained from the UCI Machine Learning Repository \cite{Dua}. From this dataset, we extract 51,630 samples (same quantity as the actual test set) to simulate an online learning procedure. Each sample consists of the audio features and the release year of the corresponding song. 
A total of 90 features, namely 12 timbre averages and 78 covariances, are extracted per sample from The Echo Nest API. As a preprocessing step, each feature is normalized so that average of them and their squares are 0 and 1, respectively. This process speeds up the convergence rate for all competitors in an unbiased manner. Additionally, the years are known to 
peak in quantity around 2000s, thus they are shifted by $-2000$ before optimization. The absolute deviation error is used as the loss at each time such that we incur the loss ${error(t) = |year.estimate(t) - actual.year(t)|}$. We use the linear regression model ${year.estimate(t) = \decision_t^T x_t}$ where $w_t$ is the decision and $x_t$ is the feature at $t$ (not necessarily observed) with an appended bias factor of $1$. The decision space is restricted to the smallest origin-centered ball with a radius of an integer power of $2$ that contains the analytical least squares solution (as opposed to the least absolute deviation). 

A total of 100 trials are run to simulate the online procedure.  During the odd-indexed trials, we sequentially sample a randomly ordered version of the dataset, similar to learning with SGD. During the even-indexed trials however, before the online procedure, we separate the dataset into two partitions. A song belongs to the earlier or later years partition depending on whether its release year is before 2000 or not. The samples in these partitions are also randomly ordered. Then, in a somewhat adversarial manner, the sample songs sequentially arrive either from the earlier years or the later years depending on which beta distribution component of the prior is randomly selected (in accordance with their mixture weights of $1/2$) to generate the random constant probability until the next observation. Beta parameters are, again, $\alpha_1=13, \beta_1=4$ or $\alpha_2=4, \beta_2=13$, respectively, for earlier or later years. In the end, for the task of regression we have two different experiment for the scenarios of randomized and adversarial function sequences. This divergence in experiments will show how 'APGD.EP' can perform similar to the optimal competitor, which does not need to know $p_t$ or priors but whether the function sequence is 
randomized or adversarial. 

A mean/std summary of time-averaged cumulative losses is presented in \figurename \ref{figure:song_mean_std} (similar to the simulated adversary scenario in Section \ref{subsec:simulated_adversary}) with the y-axes of error bars with right-pointing and left-pointing triangles are placed on right and left sides for randomized and semi-adversarial scenarios, respectively. 
\figurename\ref{figure:song_test_set_cumulative} and \figurename\ref{figure:song_adversarial_cumulative} displays the temporal evolution of the time-averaged accumulated losses, which are also averaged across trials for each round, for each algorithm.
\figurename\ref{figure:song_test_set_cumulative}, \ref{figure:song_adversarial_cumulative} and \ref{figure:song_mean_std} 
show us that 
when neither $p_t$ nor its prior is known, 'APGD.EP' from Algorithm \ref{algorithm:empirical} works as the next best for both the non-adversarial and semi-adversarial scenarios in \figurename \ref{figure:song_test_set_cumulative} and  \ref{figure:song_adversarial_cumulative}, respectively. The best competitor in \figurename\ref{figure:song_test_set_cumulative} is 'Ignore', which does not require any probability knowledge but holds an unfair advantage for randomized (SGD-like) procedures due to its inherent equal probability assignment, and also performs the worst under a somewhat adversarial scenario as in \figurename \ref{figure:song_adversarial_cumulative}. The best competitor in \figurename \ref{figure:song_adversarial_cumulative} is 'w/Known' and 'w/Prior' duo, which require $p_t$ and a corresponding prior, respectively. 
Regarding the variants requiring no knowledge, 'APGD.EP' works similar to 'GML' and 'Uniform' in semi-adversarial scenario of \figurename\ref{figure:song_adversarial_cumulative}, even though they hold an unfair advantage under the adversarial scenarios, as showcased in our lower bound analysis. 
However, they also performed the worst, especially 'GML', for a standard non-adversarial scenario as in \figurename \ref{figure:song_test_set_cumulative}. Consequently, when no information on the 
probabilities are present, 'APGD.EP' is the best overall choice thanks to its empirical probability estimates with convergence guarantees.

\begin{figure*}[!t]
	\centering
	\subfloat[]{
		\includegraphics[width=0.5\textwidth]{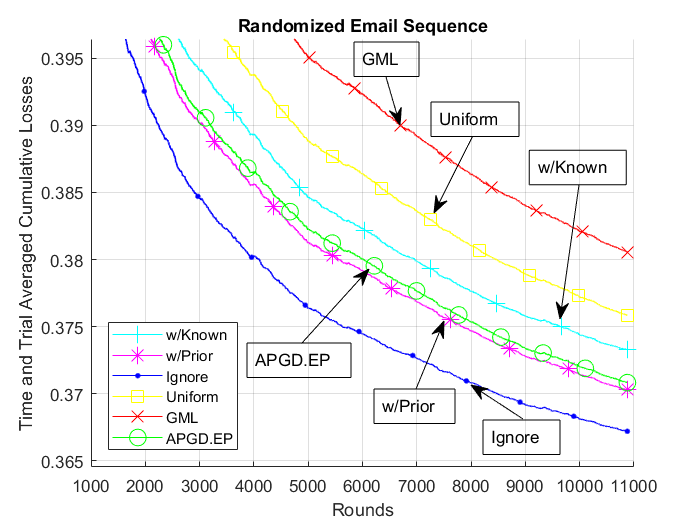}
		\label{figure:spam_random_cumulative}
	} 
	\subfloat[]{
		\includegraphics[width=0.5\textwidth]{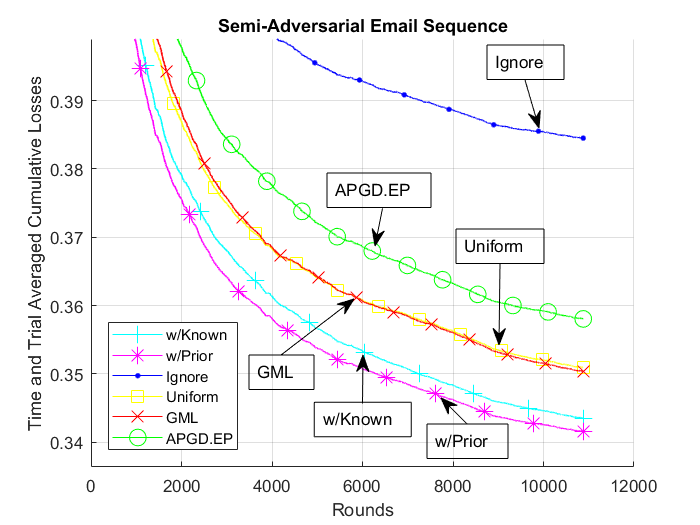}
		\label{figure:spam_adversarial_cumulative}
	}
	\newline 
	\subfloat[]{
		\includegraphics[width=0.5\textwidth]{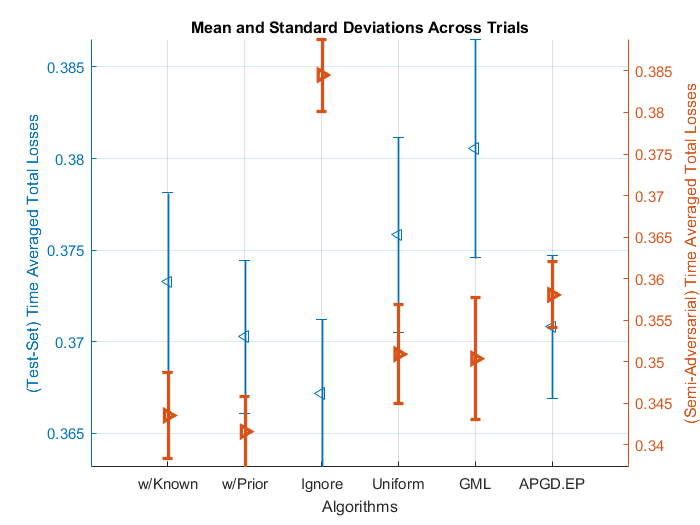}
		\label{figure:spam_mean_std}
	}
	\caption{Time-averaged cumulative logistic losses of 6 algorithms for spam email classification from Spambase Data Set for a standard and an adversarial scenarios: (a) Non-adversarial scenario with emails arriving randomly without pattern (b) Semi-adversarial scenario with emails between successive observations all either spam or not. (c) Mean and SD of the total losses from 6 algorithms for 2 cases.}
\end{figure*}
\begin{figure*}[!t]
	\centering 
	\subfloat[]{
		\includegraphics[width=0.5\textwidth]{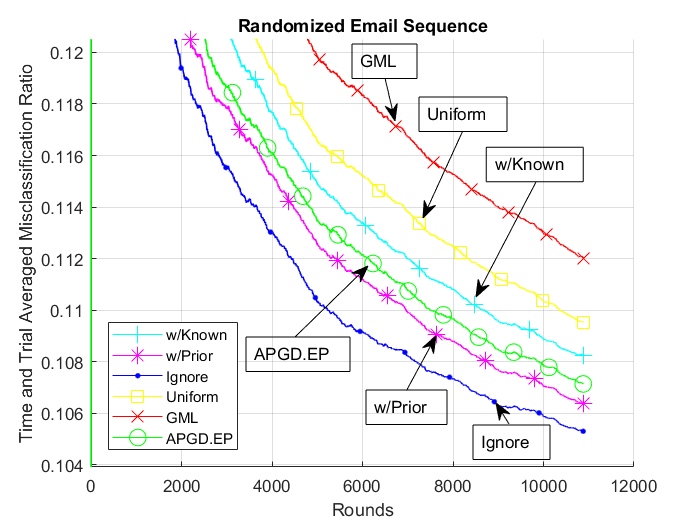}
		\label{figure:spam_random_misclass}
	} 
	\subfloat[]{
		\includegraphics[width=0.5\textwidth]{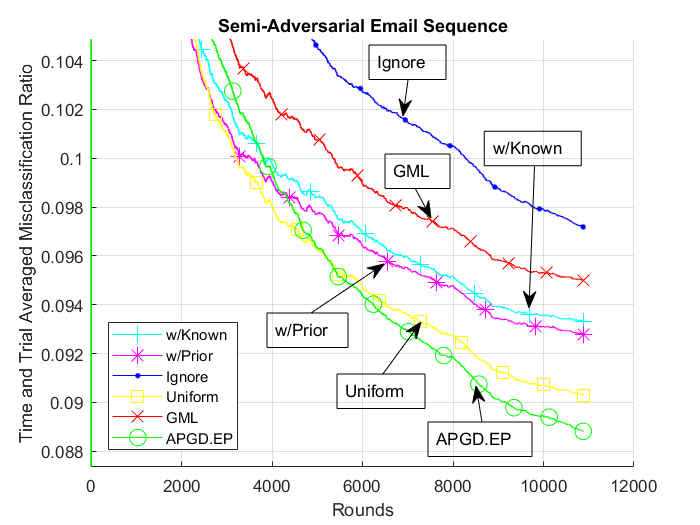}
		\label{figure:spam_adversarial_misclass}
	}
	\newline 
	\subfloat[]{
		\includegraphics[width=0.5\textwidth]{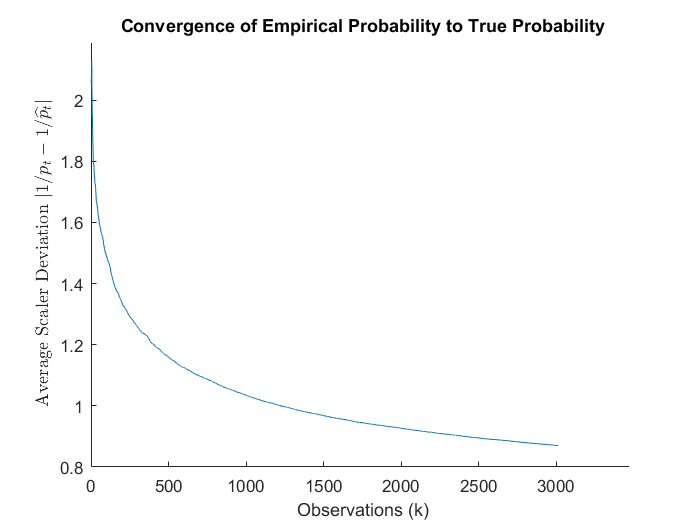}
		\label{figure:empiric_converge}
	}
	\caption{Average misclassification ratio of 6 algorithms for spam email classification from Spambase Data Set for (a) standard and (b) adversarial scenarios. Also, (c) convergence of empirical probability in time and trial averaged $|1/p_t - 1/\widehat{p_t}|$ as observations are made.}
\end{figure*}

\subsection{Classification as Logistic Regression for Spam Detection}
In this experiment, we investigate the problem of classification using logistic regression. We use the Spambase Dataset obtained from the UCI 
ML Repository \cite{Dua}. It contains 4601 samples (emails), with 57 features, 54 of them
describing the frequencies of select
48 
words and 
6 
characters in an email. Remaining 3 features are the average and longest length of an uninterrupted capital letter sequence, and the total number of capitalized letters in an email. The target (binary label) $y_t$ describes whether an email is spam. 
The loss is defined as 
$error(t) = - y_t\log(h(x_t;\decision_t)) - (1-y_t)\log(1-h(x_t;\decision_t))$,
where 
$\decision_t$ is our decision, $x_t$ is the feature vector of the email received at that time and ${h(x_t;\decision_t)}$ effectively becomes the probability of being spam estimated by $\decision_t$ to an email with feature vector $x_t$ such that  ${h(x_t;\decision_t) = (1+\exp(-\decision_t^T x_t))^{-1}}$. 

Similar to the previous experiment, each feature is normalized and the constant $1$ is appended to each $x_t$. 
The decision space is restricted to the unit $L^2$ ball. 
We extract from the dataset 10878 samples to better visualize the temporal evolution and guarantee that there will always be samples to supply for the semi-adversarial scenario 
(similar to the one in Section \ref{subsec:linear_regression}), 
by concatenating 6 copies of the dataset.
Then, like before, 
100 trials are run to simulate the online procedure and for the odd-indexed trials, the dataset is randomly ordered, while for the even-indexed trials, the dataset is partitioned as spam or not, both being randomly ordered themselves. Like the song experiment, for the even trials, the email is spam or regular 
depending on the randomly selected beta distribution, which has generated the constant probability until the next observation. As before, the 
parameters are $\alpha_1=13, \beta_1=4$ or $\alpha_2=4, \beta_2=13$, respectively, for spam and regular. 

\figurename\ref{figure:spam_random_cumulative}, \ref{figure:spam_adversarial_cumulative} and \ref{figure:spam_mean_std} collectively show that 'APGD.EP' performs very reliably. In the randomized email sequence scenario shown in \figurename\ref{figure:spam_random_cumulative}, it nears the performance of 'Ignore', which has an unfair advantage for that scenario, and against the minimax optimal algorithms 'w/Known' and 'w/Prior', which require information of observation probability $p_t$ or its prior, it performs similar to the higher performing one. In the semi-adversarial email sequence scenario shown in \figurename\ref{figure:spam_adversarial_cumulative}, it again nears the performance of 'Uniform' and 'GML', which hold 
an unfair advantage for this scenario, similar to 'Ignore' with the randomized sequence. Its convergence also nears the minimax optimal algorithms again. The means and standard deviations shown in \figurename\ref{figure:spam_mean_std} also display similar comparisons. 

In addition to these logistic loss analysis, we also detail in \figurename\ref{figure:spam_random_misclass} and \figurename\ref{figure:spam_adversarial_misclass} the misclassification performances. Interestingly, although comparisons here liken to the loss versions, the exceptions would be the increased relative performance of 'Uniform' and 'APGD.EP' with 'APGD.EP' outperforming. These results, combined with the convergence of scalers ($1/p_t$ estimate) by the empirical probabilities, as shown in \figurename\ref{figure:empiric_converge}, further strengthen our claim to utilize 'APGD.EP' in the absence of information about the stochastic observability.

\section{Conclusion}
\label{sec:conclusion}
In this paper, we have investigated the online convex optimization scenario where an optimizer may noisily observe the environmental vectors, i.e. sub-gradients, during an online sequential convex optimization procedure. We even considered the case where this noise can cause complete 
loss of sub-gradient feedback. 
After constructing the unbiased estimators of these noise corrupted or stochastically lost sub-gradients, 
we have devised adaptive algorithms with minimax optimality guarantees for both arbitrary zero-mean noises and complete stochastic observation failures. To construct the estimators, 
we either use the observation probability at each optimization round or 
derive (estimate) said probability by knowing some prior used in its generation following the previous observation, depending on the knowledge available. 
Finally, we have derived an algorithm, which required no external knowledge on the stochastic properties of feedback failures, by utilizing the empirical distribution function and directly estimating the needed conditional probability of observation, which also covered the prior generation scenario for a time-invariant prior. We tested these algorithms under the stochastic feedback failure setting for both randomized and adversarial target function sequences. 
The experiments demonstrate that the 'APGD.EP' (Alg.\ref{algorithm:empirical}), the empirical approach, could perform similar to the optimal algorithms with knowledge on the nature of function sequence or the feedback failure probabilities. 
\bibliographystyle{IEEEtran_bst/IEEEtran}
\bibliography{IEEEtran_bst/IEEEabrv,bibliography}

\end{document}